\documentclass[10pt,letterpaper]{article}

\usepackage{iccv}
\usepackage{times}
\usepackage{epsfig}
\usepackage{graphicx}
\usepackage{amsmath}
\usepackage{amssymb}
\usepackage{amsthm}
\usepackage[margin=2in]{geometry}

\newtheorem{proposition}{Proposition}
\newtheorem{definition}{Definition}

\usepackage[pagebackref=true,breaklinks=true,letterpaper=true,colorlinks,bookmarks=false]{hyperref}

\iccvfinalcopy 

\def\iccvPaperID{1206} 

\ificcvfinal\fi

\begin{document}

\title{Finding Mirror Symmetry via Registration}

\author{Marcelo Cicconet, David G. C. Hildebrand, and Hunter Elliott\\
{\small Image and Data Analysis Core, Harvard Medical School, Boston, MA}
}

\maketitle

\begin{abstract}
Symmetry is prevalent in nature and a common theme in man-made designs. Both the human visual system and computer vision algorithms can use symmetry to facilitate object recognition and other tasks. Detecting mirror symmetry in images and data is, therefore, useful for a number of applications. Here, we demonstrate that the problem of fitting a plane of mirror symmetry to data in any Euclidian space can be reduced to the problem of registering two datasets. The exactness of the resulting solution depends entirely on the registration accuracy. This new Mirror Symmetry via Registration (MSR) framework involves (1) data reflection with respect to an arbitrary plane, (2) registration of original and reflected datasets, and (3) calculation of the eigenvector of eigenvalue -1 for the transformation matrix representing the reflection and registration mappings. To support MSR, we also introduce a novel 2D registration method based on random sample consensus of an ensemble of normalized cross-correlation matches. With this as its registration back-end, MSR achieves state-of-the-art performance for symmetry line detection in two independent 2D testing databases. We further demonstrate the generality of MSR by testing it on a database of 3D shapes with an iterative closest point registration back-end. Finally, we explore its applicability to examining symmetry in natural systems by assessing the degree of symmetry present in myelinated axon reconstructions from a larval zebrafish.
\end{abstract}

\section{Introduction}
Symmetry is frequently found in nature and man-made designs. 
The human visual system exploits this fact to facilitate object recognition \cite{Vetter1994}. 
Similarly, computational tools can take advantage of symmetry for simplified data representation because it implies a considerable degree of information redundancy \cite{Gens2014}. 
Therefore, symmetry detection has great potential utility in practical computer vision applications including object recognition and image compression.

In natural images, any present symmetric objects are often surrounded by clutter or partially occluded. 
This makes symmetry detection challenging, forcing methods to be robust to outliers.
Perhaps as a consequence, the symmetry detection approaches that currently perform best (\cite{Loy2006}, \cite{ConvSymm2016}) are partially or entirely based on sampling or voting schemes. 
However, despite increased resiliency through such schemes, the current state-of-the-art methods still leave substantial room for improvement.

In this work, we show that the problem of finding mirror symmetry---also known as reflection symmetry or bilateral symmetry---in $\mathbb{R}^n$ can be reduced to a registration problem using a new method that we refer to as Mirror Symmetry via Registration (MSR).
This is accomplished by computing the eigenvector of eigenvalue $-1$ for a transformation matrix computed from reflection and registration mappings. 
We provide straightforward theoretical deductions to support this claim and demonstrate its utility through examples. 

To enhance symmetry detection with MSR, we also present a registration algorithm of the random sample consensus (RANSAC) class for two-dimensional (2D) images.
This algorithm infers optimal parameters from a collection of patch-to-image registrations computed via Normalized Cross-Correlation (NXC) \cite{Lewis95}. 
By combining these approaches, we achieve state-of-the-art performance for 2D symmetry line and segment detection on two independent testing databases: the CVPR 2013 Symmetry Detection from Real World Images competition \cite{Liu2013} and the NYU Symmetry database \cite{ConvSymm2016}.

To highlight the MSR procedure's generality, we also applied it to symmetric three-dimensional (3D) objects from the McGill 3D Shape Benchmark \cite{Siddiqi2008}. 
These tests achieved 86\% accuracy when using an Iterative Closest Point (ICP) algorithm for the underlying registration \cite{Chen1992,Besl1992}.

Finally, we recognize that symmetry detection algorithms can be helpful for analyzing natural systems.
Toward this goal, we apply MSR to start examining questions in the field of neuroscience that are concerned with the degree of symmetry in neuronal morphology and patterning.

The approach we present here exposes symmetry detection to a new line of attack and introduces registration technique developers to a new type of data on which to test novel methods. 
The paper is organized as follows: 
Section~\ref{sec:litrev} reviews existing literature to put MSR in perspective. 
Section~\ref{sec:method} contains mathematical preliminaries and a description of MSR. 
Section~\ref{sec:experiments} details quantitative experiments on 2D and 3D testing databases. Section~\ref{sec:application} describes applications in neuroscience. 
Finally, section~\ref{sec:conclusion} discusses MSR advantages, limitations, and possible future directions.

\begin{figure*}[p]
\centering
\begin{minipage}[c]{0.4\linewidth}
	\flushleft (a)
	\includegraphics[width=\linewidth]{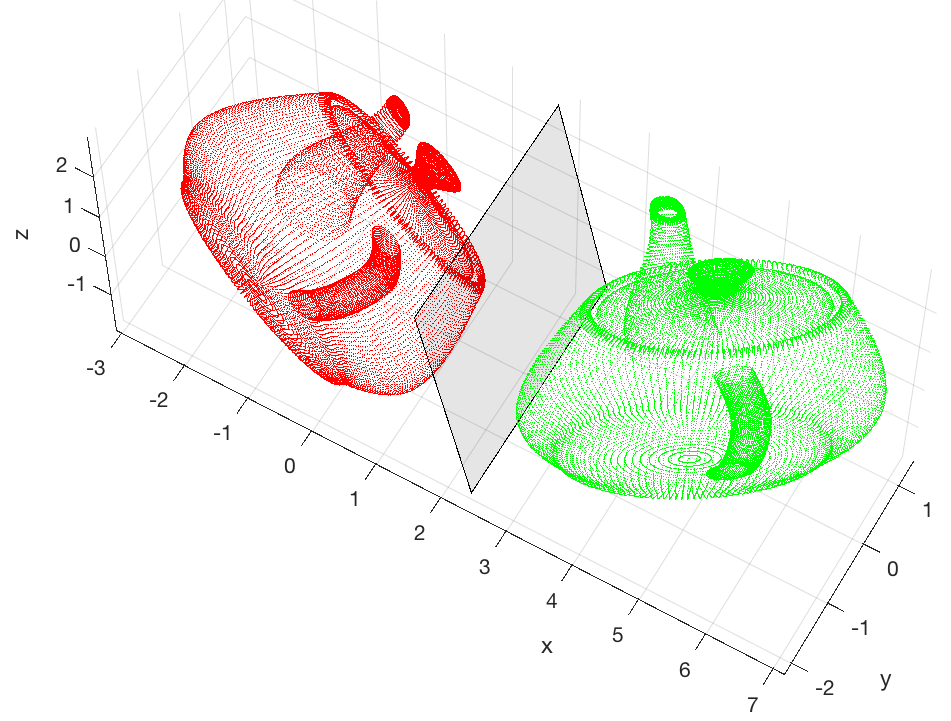}
\end{minipage}
\begin{minipage}[c]{0.4\linewidth}
	\flushleft (b)
	\includegraphics[width=\linewidth]{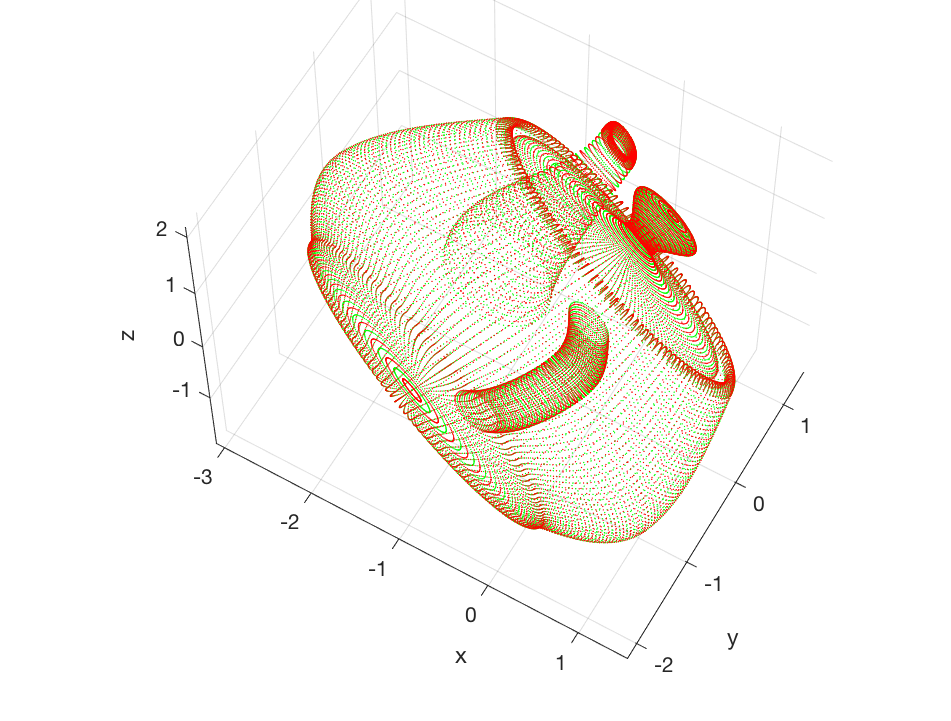} 
\end{minipage}\\
\begin{minipage}[c]{0.4\linewidth}
	\flushleft (c)
	\includegraphics[width=\linewidth]{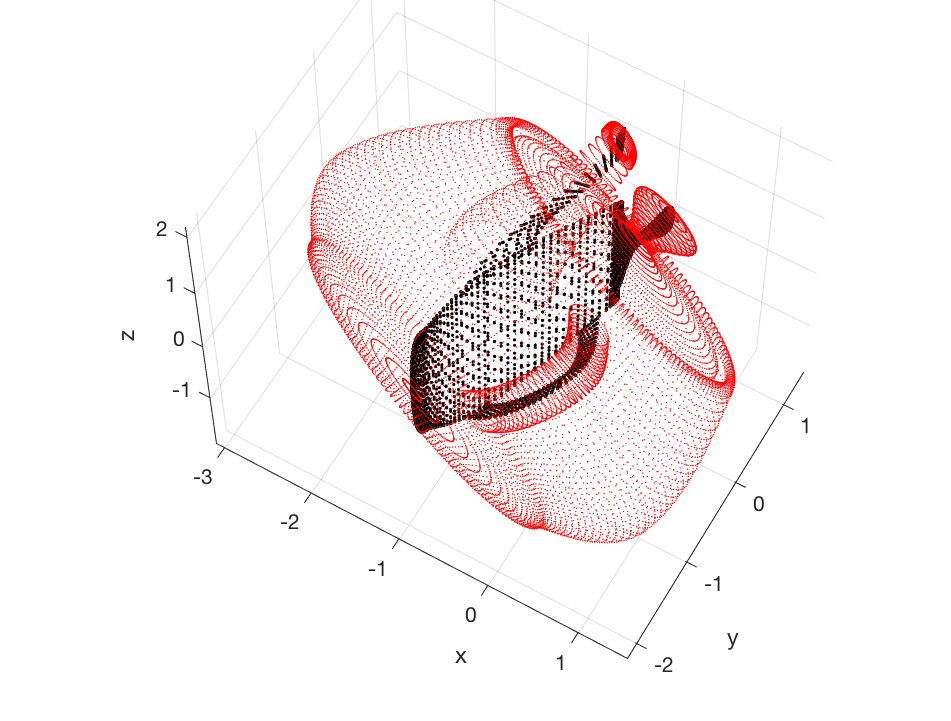}
\end{minipage}
\begin{minipage}[c]{0.4\linewidth}
	\flushleft (d)
	\includegraphics[width=\linewidth]{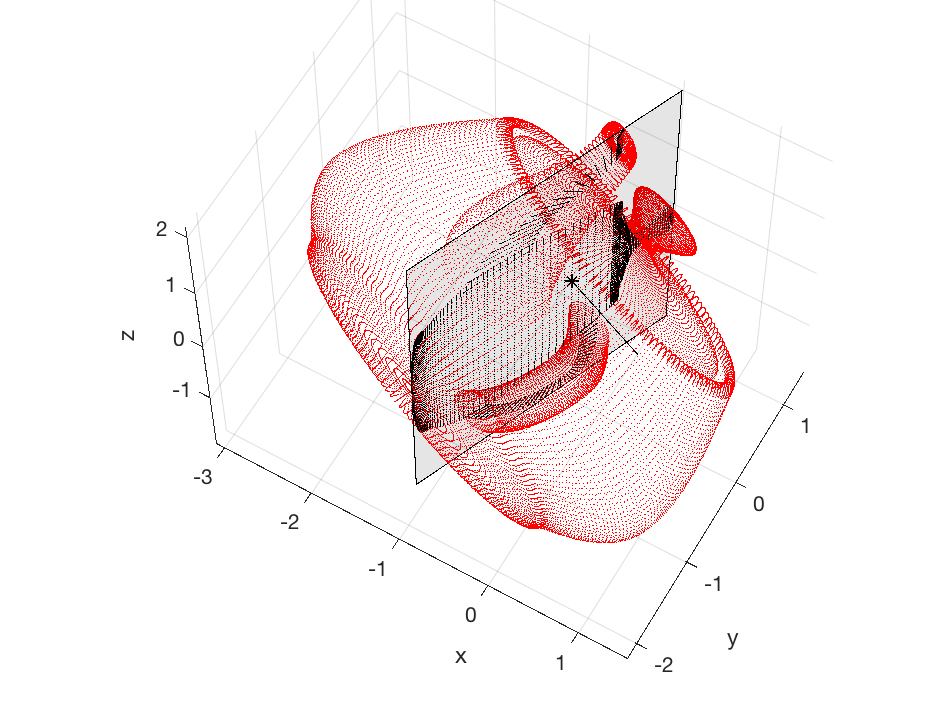}
\end{minipage}\vspace{0.1cm}
\caption{Symmetry plane detection in 3D using MSR. 
(a) The original data (red) is reflected (green) with respect to an arbitrary plane. 
(b) Registration between original and reflected point clouds using an Iterative Closest Point (ICP) algorithm. 
(c) Visualization of every midpoint (black) between a point in the original set and the corresponding point in the reflected set. 
(d) Computing the symmetry plane in this case involves either fitting a plane to the midpoints in (c) or analytically solving an eigenvalue problem on a function of the transformation matrices corresponding to the reflection and ICP registration mappings.}
\label{fig:symplanedet}
\end{figure*}

\section{Previous Work}\label{sec:litrev}

Listed here is a review of some relevant previous work.

\subsection{Mirror Symmetry Detection in 2D}

\noindent \textbf{2006.} \cite{Loy2006}: Scale-invariant feature transform (SIFT) features were grouped into ``symmetric constellations'' by a voting scheme. 
Dominant symmetries present in the image emerged as local maxima.

\noindent \textbf{2012.} \cite{Lee2012}: Generalized mirror symmetry detection to a curved transflection (glide reflection) symmetry detection problem. Estimated symmetry via a set of contiguous local straight reflection axes.

\noindent \textbf{2013.} \cite{Kondra2013}: Developed a 3-step algorithm, wherein (1) SIFT correlation measures are computed along discrete directions, (2) symmetrical regions are identified from matches in the directions characterized by maximum correlations, and then steps (1) and (2) are repeated at different scales. 
\cite{Patraucean2013}: Created a 2-step algorithm, wherein candidates for mirror-symmetric patches are identified using a Hough-like voting scheme and then validated using a principled statistical procedure inspired from \textit{a contrario} theory.
\cite{Michaelsen2013}: Introduced a combinatorial gestalt algebra technique to be used on top of SIFT descriptors. 
\cite{Liu2013}: Evaluated the performance of various symmetry detection methods on a common database, with \cite{Loy2006} emerging as overall winner.

\noindent \textbf{2014.} \cite{Cicconet2014CVPR}: Described a pairwise voting-scheme based on tangents computed via wavelet filtering.
\cite{Cai2014}: Presented an adaptive feature point detection algorithm to overcome susceptibility to clutter in feature-based methods.

\noindent \textbf{2015.} \cite{Wang2015}: Exhibited use of traditional edge detectors and a voting process, respectively, before and after a novel edge description and matching step based on locally affine-invariant features.

\noindent \textbf{2016.} \cite{ConvSymm2016}: Introduced a pairwise convolutional approach to mirror symmetry detection similar to \cite{Cicconet2014CVPR}. 
The method outperformed \cite{Loy2006} by a small margin and its authors released a new database, which we use here for testing.
\cite{Funk2016}: Exploited ambiguities and challenges in symmetry detection to propose a method for producing reCAPTCHA solutions based on symmetry.

\subsection{Mirror Symmetry Detection in 3D}

\noindent \textbf{1997.} \cite{Sun1997}: Converted the symmetry detection problem to the correlation of the Gaussian image.

\noindent \textbf{2002.} \cite{Benz2002}: Presented an approach similar to MSR in the reflection and registration steps.
However, the symmetry plane was fit on the set of midpoints, not obtained as the eigenvalue solution.
Mathematical proofs for the results were not provided and tests were only conducted in 3D for symmetry in human faces.

\noindent \textbf{2006.} \cite{Podolak2006}: Described a planar reflective symmetry transform that captures a continuous measure of a shape's degree of mirror symmetry with respect to all possible planes. 
\cite{Pan2006}: Presented a more robust Gaussian image-based approach.

\noindent \textbf{2011.} \cite{Wu2011}: Introduced a 2-step method, wherein landmark-related region detection is followed by a learning stage that computes a symmetry plane from the landmarks based on training input consisting of standard symmetry planes identified by medical experts. 
\cite{Padia2011}: Reviewed \cite{Benz2002} and ICP variations.

\noindent \textbf{2013.} \cite{Mitra2013}: Discussed applications in computer graphics and geometry that can utilize symmetry information for more effective processing. \cite{Kakarala2013}: Described bilateral symmetry plane estimation for 3D shapes that is carried out in the spherical harmonic domain.

\noindent \textbf{2014.} \cite{Sipiran2014}: Presented an algorithm that generates a set of candidate symmetries by matching local maxima of a surface function based on heat diffusion in local domains, with a global optimum obtained by a voting scheme.

\noindent \textbf{2015.} \cite{Zheng2015}: Developed a skeleton-intrinsic symmetrization method for recovering the aesthetics of mirror symmetry from asymmetric shapes while preserving their general pose. This was accomplished by measuring intrinsic distances over a curve skeleton backbone for symmetry analysis, symmetrizing about the skeleton, and propagating the symmetrization from skeleton to shape.

\noindent \textbf{2016.} \cite{Li2016}: Achieved symmetry plane detection by generating a candidate plane based on a matching pair of sample views and then verifying whether the number of remaining matching pairs fell within a preset minimum number.

\subsection{Mirror Symmetry Detection in $\mathbb{R}^n$}

We are unaware of any previous work claiming to devise a mirror symmetry detection method that is invariant to the dimension of the examined space.

\section{Method}\label{sec:method}

\begin{definition}
[of Mirror Symmetry]
A set of points $P \subset \mathbb{R}^n$ is said to present \emph{mirror}, \emph{reflection}, or \emph{bilateral} symmetry if there exists a hyperplane $H \subset \mathbb{R}^n$ of dimension $n-1$ such that the \emph{mirror reflection} of $P$ with respect to $H$ produces a set of points $Q$ such that $P = Q$.
\end{definition}

\begin{definition}
[of Mirror Reflection]
Let $H \subset \mathbb{R}^n$ be a $(n-1)$-dimensional hyperplane, $v$ a unit vector perpendicular to $H$, and $p$ a fixed point in $H$, so that $H = \{q \in \mathbb{R}^n : \langle q-p, v\rangle = 0\}$.
The \emph{mirror reflection} of a set of points $P$ with respect to $H$ is the set $\{q-2 \langle q-p,v \rangle v : q \in P\}$.
\end{definition}

The mirror reflection with respect to a plane through the origin and with normal vector $v$ is given by $x \mapsto S_v x$, where $S_v = I-2vv^\top$, where $I$ is the identity matrix.
The reflection with respect to a plane through an arbitrary point $p$ and with normal vector $v$ is given by:
$$x \mapsto S_{p,v}(x) = S_vx+2dv \text{ ,}$$
\noindent where $d = \langle p , v \rangle$ is the ``signed'' distance between the plane and the origin. For simplicity of notation, we will henceforth denote $S_{p,v}(x)$ as $S_{p,v}x$.

The symmetry plane in $\mathbb{R}^n$ can be computed in 3 steps, as illustrated in Figure~\ref{fig:symplanedet}:
\begin{enumerate}
\item Reflect original data with respect to an arbitrary plane.
\item Register original and reflected sets.
\item Infer optimal symmetry plane from the parameters of the reflection and registration mappings.
\end{enumerate}

\subsection*{Remarks}

\noindent (a) Depending on the registration algorithm used, it can help to start in MSR step 1 with an arbitrary plane that is near---or a good guess for---the actual symmetry plane. 
We employ this strategy for the application described in Section~\ref{sec:application}.
Alternatively, several runs with different initial planes can be attempted and the one for which the registration algorithm returns the most confident result chosen.
We use this second strategy for the 3D experiments in Section~\ref{sec:experiments}.\vspace{0.2cm}

\begin{figure*}[!h]
\vspace{0.5cm}
\centering
\begin{minipage}[c]{0.19\linewidth}
	\centering
	\includegraphics[width=\linewidth]{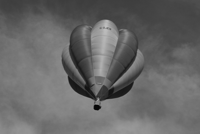}\\(a)
\end{minipage}\hfill
\begin{minipage}[c]{0.19\linewidth}
	\centering
	\includegraphics[width=\linewidth]{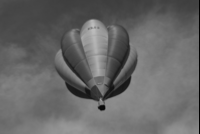}\\(b)
\end{minipage}\hfill
\begin{minipage}[c]{0.19\linewidth}
	\centering
	\includegraphics[width=\linewidth]{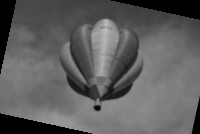}\\(c)
\end{minipage}\hfill
\begin{minipage}[c]{0.19\linewidth}
	\centering
	\includegraphics[width=\linewidth]{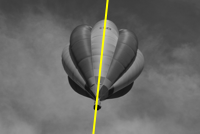}\\(d)
\end{minipage}\hfill
\begin{minipage}[c]{0.19\linewidth}
	\centering
	\includegraphics[width=\linewidth]{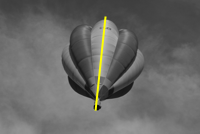}\\(e)
\end{minipage}
\vspace{0.3cm}
\caption{Symmetry line and segment detection in 2D using MSR. 
(a) Input image. 
(b) Mirror reflection of the original image with respect to a vertical line through the center of the image. 
(c) Registration of (b) with respect to (a). 
(d) The symmetry line computed by the proposed algorithm (MSR) is shown in yellow. 
(e) The symmetry segment can be computed from (d) via a post-processing phase, as in \cite{ConvSymm2016}.}
\label{fig:symlinedet}
\end{figure*}

\noindent (b) All steps in the MSR framework are exact (when factoring out numerical errors) except for registration.
If the data is not perfectly mirror symmetric, then the registration will not be precise. 
The MSR approach reduces mirror symmetry detection to a registration problem, with the caveat that its robustness depends entirely on the robustness of the underlying registration method.\vspace{0.2cm}

\noindent (c) MSR step 3 can be performed in one of two ways: either by fitting a plane through the midpoints of all corresponding original-transformed point pairs, or by solving an eigenvalue problem related to the global (reflection and rigid) transformation that was applied to the original data during registration. We adopt the latter approach here.\vspace{0.2cm}

We now mathematically demonstrate why the MSR approach works for detecting mirror symmetry.

Let $P = \{p_1,...,p_N\}$ be a point cloud and $Q = \{q_1,...,q_N\}$ the reflection of $P$ given by $S_{p,v}$, that is, $q_i = S_{p,v}p_i \, \forall i$.

\begin{proposition}
Let $m_i = \frac{1}{2}(p_i+q_i)$, so that $M = \{m_1,...,m_N\}$ is the set of midpoints between corresponding points in $P$ and $Q$. Then the set $M$ is contained in the plane with normal vector $v$ passing through $dv$.
\label{prop1}
\end{proposition}

\begin{proof}
For $x \in P$, the reflection by $S_{p,v}$ is $S_v x +2dv$, so
\begin{eqnarray}
\langle \frac{1}{2} (x + S_v x +2dv) -dv , v\rangle =\\
\langle \frac{1}{2} x,v \rangle + \langle \frac{1}{2} S_v x, v \rangle + \langle dv, v \rangle - \langle dv , v\rangle.
\label{eq:sym}
\end{eqnarray}
But $S_v$ is symmetric, so $\langle S_v x,v \rangle =  \langle x, S_v v \rangle$.
Further, $S_v v = -v$ because $S_v$ is the reflection with respect to the plane through the origin with normal vector $v$, so $\langle x, S_v v \rangle = -\langle x, v \rangle$. 
Therefore, (\ref{eq:sym}) is equal to $0$. 
\end{proof}

Let now $R$ be the rigid transformation defined by $R(x) = R_0x+t$, where $R_0$ is a rotation matrix and $t$ a translation vector. 
If we reflect a point $x \in P$ by $S_{p,v}$ and then transform it through $R$, the result is $R_0(S_vx+2dv)+t$.

\begin{proposition}
Let $T = S_vR_0^\top$ and $w$ equal the unit eigenvector of $T$ corresponding to the eigenvalue $-1$. 
That is, $Tw = -w$. 
We will show in \emph{Proposition 3} that such a $w$ exists. 
Let $r = \frac{1}{2}(R_0(2dv)+t)$, with $d$ as previously defined. 
Then the midpoints $\frac{1}{2}(x+R_0(S_vx+2dv)+t)$ lie in the plane with normal vector $w$ passing through $r$.
\end{proposition}

\begin{proof}
\begin{eqnarray}
\langle \frac{1}{2}(x+R_0(S_vx+2dv)+t)-r , w \rangle =\\
\langle \frac{1}{2}(x+R_0(S_vx+2dv)+t)-\frac{1}{2}(R_0(2dv)+t), w \rangle =\\
\frac{1}{2} \langle x+R_0(S_vx), w \rangle =\\
\frac{1}{2} (\langle x, w \rangle + \langle R_0S_vx, w \rangle) =\\
\frac{1}{2} (\langle x, w \rangle + \langle x, S_vR_0^\top w \rangle) =\\
\frac{1}{2} (\langle x, w \rangle + \langle x, -w \rangle) =\\
0.
\end{eqnarray}
\end{proof}

\begin{proposition}
If $S$ is a reflection and $R$ is a rotation, then $SR$ is a reflection.
As a consequence, $SR$ necessarily has $-1$ as an eigenvalue.
\end{proposition}

\begin{proof}
Since $S$ and $R$ are orthogonal, so is $SR$.
Further, since the determinant of the product equals the product of the determinants, the determinant of $SR$ is $-1$. 
Considering the reflection $SR$, let $H$ be the reflection hyperplane with normal vector $v$.
In this case, $SRv = -v$, so $v$ is the eigenvector of $SR$ corresponding to the eigenvalue $-1$.
\end{proof}

\paragraph{Algorithm.}
Given these results, the precise MSR algorithm for finding the mirror symmetry plane for a set of points $P$ is:
\begin{enumerate}
\item Choose an initial reflection plane, given by a point $p$ and a perpendicular vector $v$.
For example, $p$ can be the average (center of mass) of the points in $P$ and $v$ set as the vector $(1,0,...,0)$.
\item Reflect all the points $x \in P$:
$$d = \langle p,v\rangle\text{ ,}$$
$$x \mapsto S_{p,v}x = S_vx+2dv\text{ ,}$$
to obtain a new set of points $Q$.
\item Register $Q$ to $P$ through a rigid transformation, obtaining a rotation matrix $R_0$ and a translation vector $t$.
That is, the registration transformation is given by $x \mapsto R_0x+t$.
\item Compute the eigenvector $\bar{v}$ of the matrix $S_vR_0^\top$ corresponding to the eigenvalue $-1$, where $\bar{v}$ is the vector perpendicular to the symmetry plane.
\item Completely define the symmetry plane by computing a point
$\bar{p}$ through which it passes:
$$\bar{p} = \frac{1}{2}(R_0(2dv)+t)\text{ .}$$

\end{enumerate}

\section{Quantitative Experiments}\label{sec:experiments}

\subsection{Accuracy Metric}

For 2D cases, we examined MSR accuracy using established metrics. 
When detecting symmetry segments, the metric described in \cite{Liu2013} was used. 
When detecting symmetry lines, an extension of the metric for segments \cite{ConvSymm2016} was used. 
In brief, the correctness criteria for segments was based on both angle and center proximity between the prediction result and the ground truth. 
The correctness criteria for lines was similar, except that center proximity was replaced with the distance between the center of the ground truth and the prediction line because the prediction line has no defined center.
For thorough evaluation across approaches, precision/recall curves were generated for each method from up to the top ten results.

For 3D cases, we were unable to find general-purpose databases or accuracy metrics. Therefore, we evaluated MSR accuracy by visual inspection of projections of the data along three mutually perpendicular directions, one of which was orthogonal to the estimated symmetry plane.


\subsection{Consensus of Patch-to-Image Registrations}\label{subsec:cpti}

We initially tested MSR with a number of off-the-shelf registration algorithms in our 2D experiments: ICP on edge maps, intensity-based on images or edge maps, and speeded up robust feature (SURF)-based on images or edge maps. 
However, none of these options produced results with comparable precision/recall numbers obtained using the previous state-of-the-art symmetry detection algorithms.

Therefore, we designed a new registration method based on a consensus over an ensemble of patch-to-image registration outputs (i.e., a RANSAC approach).
First, we assume that the transformation is rigid (rotation and/or translation), which is sufficient for MSR. 
Then, for every angle $\alpha = 0, \frac{360}{N}, 2\frac{360}{N},...,(N-1)\frac{360}{N}$, we sample hundreds of square patches from the moving image and register each with respect to the target image using NXC \cite{Lewis95}, selecting only those registrations for which the maximum in correlation space is above a threshold ($\frac{1}{4}$).
For this process, we found that $N = 6$ typically provides good results. 
Finally, we look for the $K$ best local maxima in the space of registration parameters found via NXC. 
For precision/recall evaluations, we chose $K = 10$.


This NXC-based registration approach performed better in the MSR framework than the others we tested. 
In a one-shot symmetry line detection experiment on the NYU Symmetry database (176 images), our NXC-based registration achieved $95\%$ accuracy, while other methods accomplished accuracies near $73\%$ (Figure~\ref{fig:resultsOneShot}).

\begin{figure}[p!]
\vspace{-0.4cm}
\centering
\includegraphics[width=0.5\linewidth]{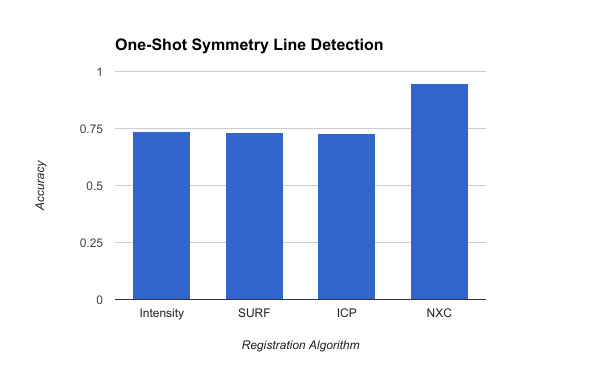}
\caption{Accuracy of one-shot symmetry line detection achieved by combining the MSR approach described in Section~\ref{sec:method} with various registration algorithms. 
Intensity- and SURF-based registrations were performed on the gradients of the images rather than images themselves, as we could not detect a significant difference between the two. 
ICP was applied on top of Ultrametric Contour Maps \cite{Arbelaez2011}. 
NXC indicates results from using the NXC-based registration described in Subsection~\ref{subsec:cpti}. 
For this experiment, the NXC-based registration was applied on image gradients for $40\times40$ patches with a maximum side length of $200$.}
\label{fig:resultsOneShot}
\end{figure}

\subsection{Results}

\subsubsection*{2D}

The previous state-of-the-art for single-symmetry line or segment detection in 2D was a pairwise convolutional method \cite{ConvSymm2016}, referred to here as Convolutional Approach to Reflection Symmetry (CARS). 
Released with its description was the database used for testing (the NYU Symmetry database), with which we tested the MSR approach.
However, given that CARS does not appear to be peer-reviewed yet, we additionally conducted testing with the CVPR 2013 database, for which Loy's method \cite{Loy2006} reported best results.

In accordance with the registration method comparison depicted in Figure~\ref{fig:resultsOneShot}, we adopted the NXC-based registration described in Subsection~\ref{subsec:cpti} to compute precision/recall curves for evaluation.
Though the MSR method outputs only \emph{lines} for 2D cases, not \emph{segments} (limited subsets of lines), the latter is required for proper comparison with CARS.
For this reason, we post-processed the symmetry line resulting from MSR into segments using a previously reported algorithm \cite{ConvSymm2016}. 
Evaluation results are shown in Figure~\ref{fig:results2D} and are accompanied by examples of symmetry line detection outputs in Figure~\ref{fig:outputs}.


\begin{figure}[p!]
\vspace{-0.2cm}
\centering
\includegraphics[width=0.5\linewidth]{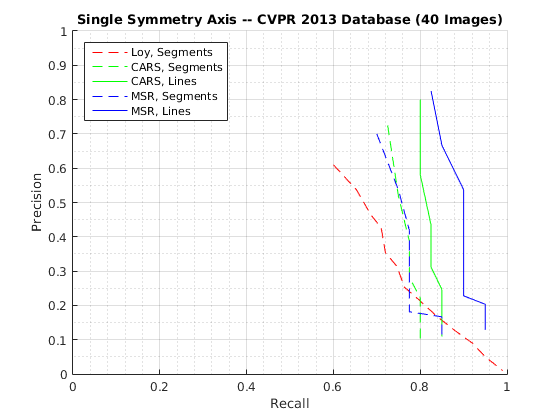}\\ \vspace{0.2cm}
\includegraphics[width=0.5\linewidth]{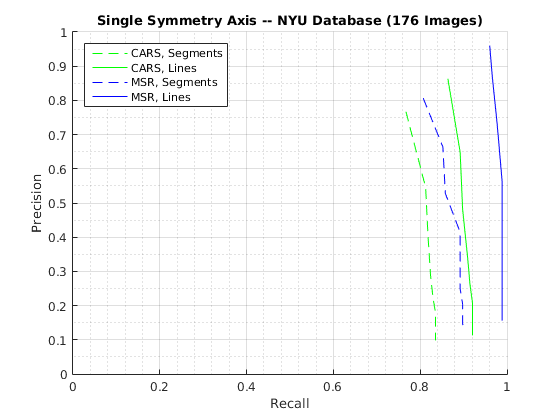}
\caption{Single-symmetry detection on the CVPR 2013 and NYU Symmetry databases. 
Loy refers to \cite{Loy2006}, which performed best in the CVPR 2013 competition.
CARS refers to \cite{ConvSymm2016}, which yielded the previous state-of-the-art, though the article is not yet peer-reviewed. 
The method we describe here is referred to as MSR (Mirror Symmetry via Registration).}
\label{fig:results2D}
\end{figure}

For 2D symmetry detection, our MSR approach outperforms the previous state-of-the-art peer-reviewed method \cite{Loy2006} for single-symmetry \emph{segment} detection on the CVPR 2013 database, while reaching similar performance as CARS (pre-print \cite{ConvSymm2016}). 
Additionally, MSR outperforms CARS for single-symmetry \emph{line} detection on this database.
Note that line detection results for \cite{Loy2006} were not reported and are therefore not available for comparison.
MSR also outperforms CARS on both segment and line detection on the NYU Symmetry database.


\begin{figure}[p!]
\centering
\begin{minipage}[c]{0.2\linewidth}
	\centering
	\includegraphics[width=\linewidth]{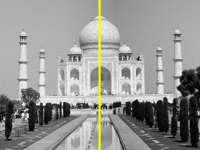}
\end{minipage}\hspace{0.5cm}
\begin{minipage}[c]{0.2\linewidth}
	\centering
	\includegraphics[width=\linewidth]{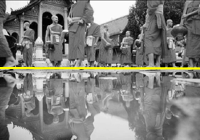}
\end{minipage}\\ \vspace{0.5cm}
\begin{minipage}[c]{0.2\linewidth}
	\centering
	\includegraphics[width=\linewidth]{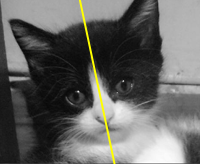}
\end{minipage}\hspace{0.5cm}
\begin{minipage}[c]{0.2\linewidth}
	\centering
	\includegraphics[width=\linewidth]{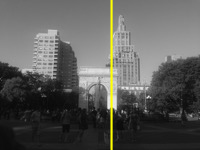}
\end{minipage}\\ \vspace{0.5cm}
\begin{minipage}[c]{0.2\linewidth}
	\centering
	\includegraphics[width=\linewidth]{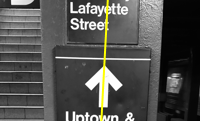}
\end{minipage}\hspace{0.5cm}
\begin{minipage}[c]{0.2\linewidth}
	\centering
	\includegraphics[width=\linewidth]{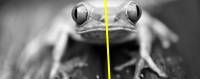}
\end{minipage}\\ \vspace{0.5cm}
\begin{minipage}[c]{0.2\linewidth}
	\centering
	\includegraphics[width=\linewidth]{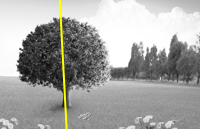}
\end{minipage}\hspace{0.5cm}
\begin{minipage}[c]{0.2\linewidth}
	\centering
	\includegraphics[width=\linewidth]{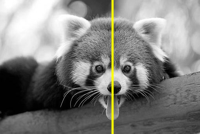}
\end{minipage}
\\ \vspace{0.5cm}
\begin{minipage}[c]{0.2\linewidth}
	\centering
	\includegraphics[width=\linewidth]{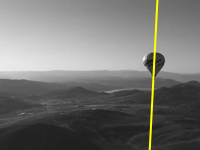}
\end{minipage}\hspace{0.5cm}
\begin{minipage}[c]{0.2\linewidth}
	\centering
	\includegraphics[width=\linewidth]{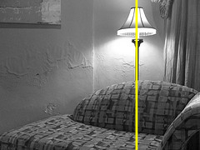}
\end{minipage}
\vspace{0.5cm}
\caption{Sample of MSR results for 2D images from the NYU Symmetry database.}
\label{fig:outputs}
\end{figure}

\begin{figure}[!h]
\centering
\includegraphics[width=0.7\linewidth]{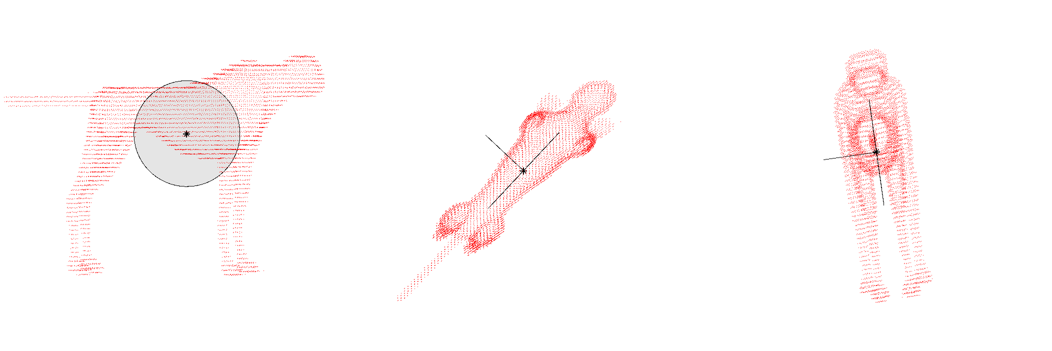}\\
\includegraphics[width=0.7\linewidth]{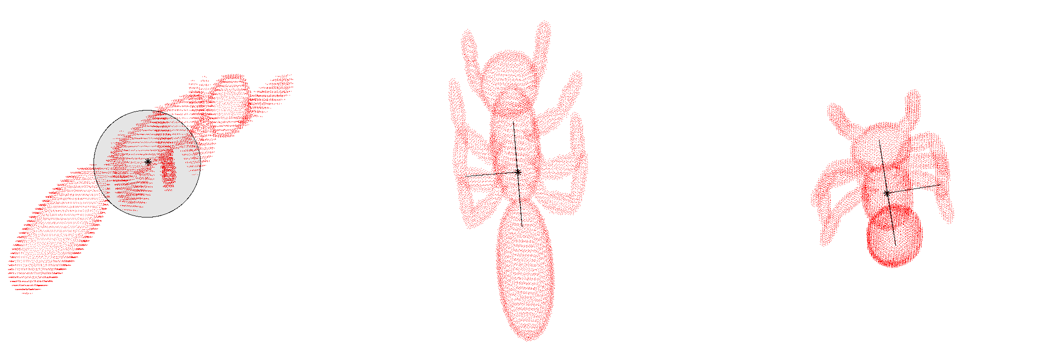}\\
\includegraphics[width=0.7\linewidth]{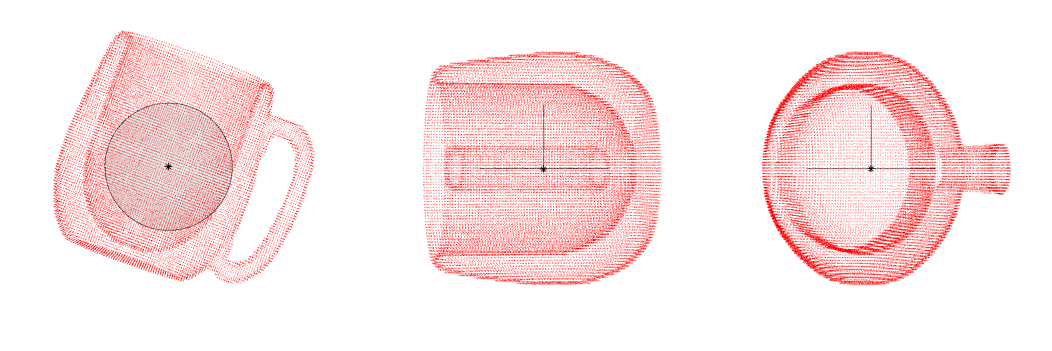}\\
\includegraphics[width=0.7\linewidth]{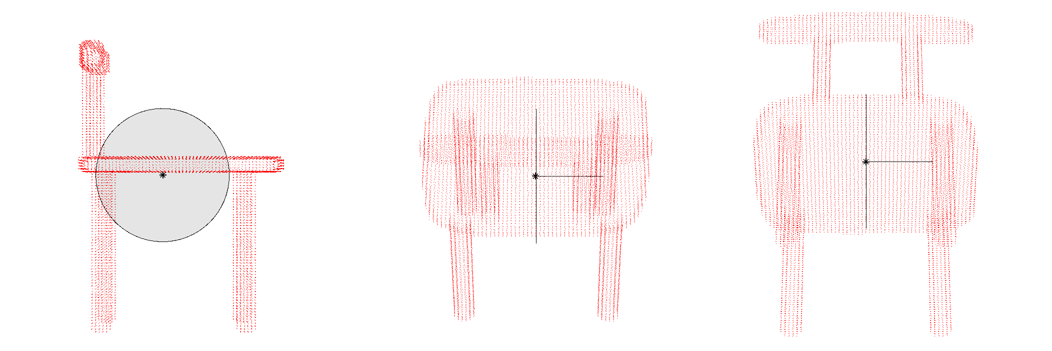}\\
 \includegraphics[width=0.7\linewidth]{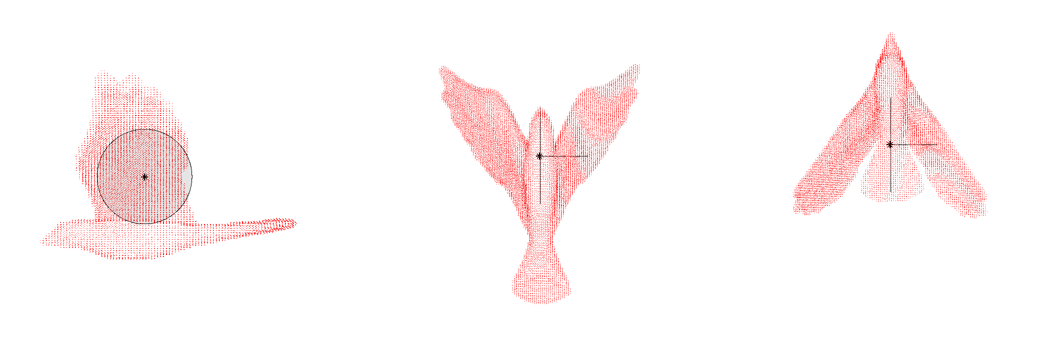}\\
\caption{Sample of MSR results for 3D shapes from the McGill 3D Shape Benchmark \cite{Siddiqi2008}. 
Each row contains the same object with three columns showing mutually perpendicular views. 
The left-most view in each row is orthogonal to the MSR-computed symmetry plane.}
\label{fig:objects}
\end{figure}

\subsubsection*{3D}

To the best of our knowledge, there are no general-purpose databases or accuracy metrics for 3D mirror symmetry detection tests. 
We thus created a testing database consisting of hand-picked 203 symmetric 3D shapes from the McGill 3D Shape Benchmark \cite{Siddiqi2008}.
The included shapes consist of surface points corresponding to objects such as cups, airplanes, and insects.
Because each shape was represented by a set of points, we chose the ICP algorithm \cite{Chen1992,Besl1992} as the registration back-end for all 3D testing.

We ran the MSR method three times per object, each with a different initial reflection hyperplane.
These hyperplanes were always selected from the set centered at the origin with perpendicular vectors given by the canonical basis $(1,0,0)$, $(0,1,0)$, $(0,0,1)$. 
As the final solution, we chose the result whose registration confidence was highest.

By visual inspection of projections along the three mutually perpendicular vectors, one of which was always orthogonal to the symmetry plane, we found that MSR achieved 86\% accuracy. 
From the set of 203 shapes, symmetry was correctly detected in 177. 
Examples are shown in Figure~\ref{fig:objects}.


\section{Application}\label{sec:application}

This project was largely driven by a practical application in the field of neuroscience. 
We were interested in investigating the degree of bilateral symmetry in myelinated axons present throughout the body of a larval zebrafish. 
Myelinated axon reconstructions were manually extracted from serial-section electron micrographs. 
The resulting data consisted of curves represented as sequences of points in 3D, which we refer to as \emph{skeletons}. 
We first sought to find the plane of bilateral symmetry given that the projections appeared nearly mirror symmetric. 
A visually acceptable result was obtained by application of the MSR approach with the ICP algorithm as the registration back-end and a manually selected initial reflection plane corresponding to the vector given by the canonical basis $(1,0,0)$.

Figure~\ref{fig:plane} illustrates the myelinated axon reconstructions and the MSR-computed symmetry plane using ICP as the registration back-end.
Figure~\ref{fig:projections} shows projection of the data along three mutually perpendicular directions, where the side projection is orthogonal to the symmetry plane.

\begin{figure}[p!]
\centering
\includegraphics[width=0.4\linewidth]{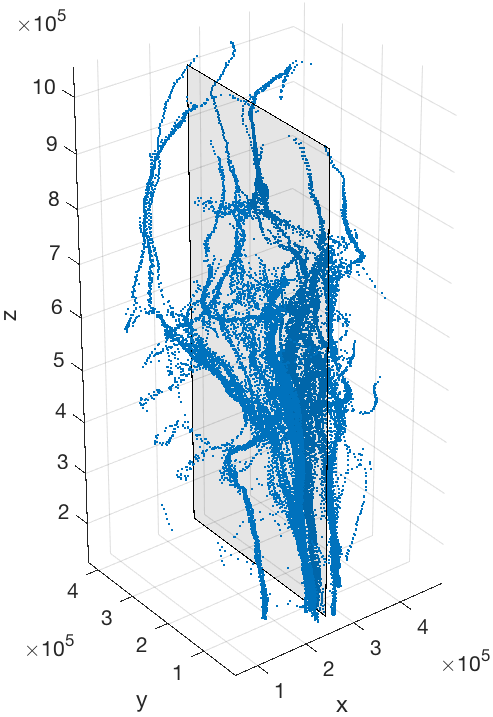}
\caption{Symmetry plane computed using the MSR approach for myelinated axon reconstructions (blue) obtained from a larval zebrafish using serial-section electron microscopy.}
\label{fig:plane}
\end{figure}

\begin{figure}[p!]
\centering
\includegraphics[width=0.5\linewidth]{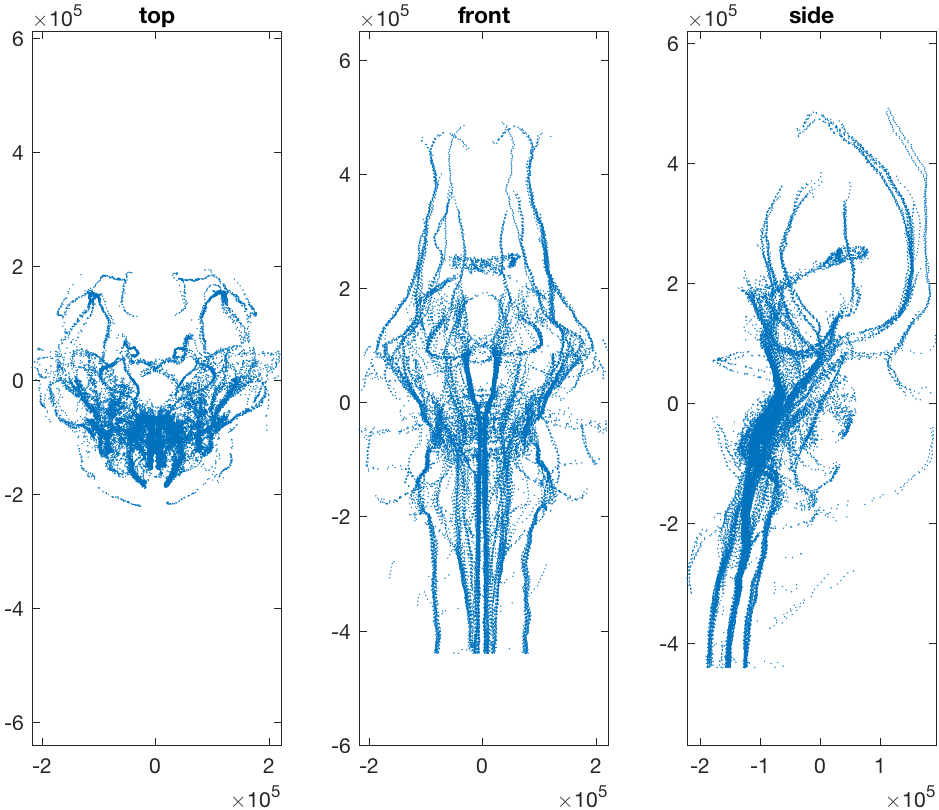}
\caption{Projections of myelinated axon reconstructions along three mutually perpendicular directions. 
The side projection is orthogonal to the MSR-computed symmetry plane.}
\label{fig:projections}
\end{figure}

We next sought to find left-right pairings of symmetric skeletons.
Each skeleton $s$ is a discrete curve in $\mathbb{R}^n$:
$$s = \{s_i : i = 1,...,n_s\}\text{ .}$$
Given two skeletons $s$ and $t$ and a reference symmetry plane $H$, a pairwise symmetry measure can be computed by comparing the original $s$ with the reflection of $t$ with respect to $H$. 
We used Dynamic Time Warping (DTW), a variation of dynamic programming that is widely used for sequence matching \cite{Li2009}, to compute a similarity cost for each original-reflection pair.
To the resulting matrix of pairwise costs $C$, where $C(i,j) = C(j,i)$ is the symmetry measure between skeletons of indexes $i$ and $j$, we applied the Munkres assignment algorithm \cite{Munkres1957} (also known as the Hungarian method) to compute a globally optimal pairwise assignment.

\section{Conclusion}\label{sec:conclusion}

In this paper, we introduced Mirror Symmetry via Registration (MSR), a new framework for mirror symmetry detection that is based on registration and invariant to dimension.
For all but the registration phase, this approach is mathematically exact.
That is, mirror symmetry detection in $\mathbb{R}^n$ is as good as the best available registration method.
In addition, we described a new 2D image registration algorithm based on RANSAC over a set of patch-to-image registrations.

To illustrate MSR performance, we provided experimental results from testing on 2D and 3D databases.
To show the utility of MSR in analyses of natural systems, we described its application to 3D symmetry detection in the myelinated axons of a larval zebrafish.

One limitation of MSR is that it does not output the intersection of the computed symmetry hyperplane with the symmetric object. 
In 2D, for example, it only outputs the symmetry line, not the symmetry segment.

Potential improvements to MSR include its extension to enable detection of multiple symmetry axes and, on the theoretical side, a metric for quantifying plane similarity in $\mathbb{R}^n$ for $n > 2$ should be developed to better measure accuracy.

{\small
\bibliographystyle{ieee}
\bibliography{egbib}

\begin{thebibliography}{10}\itemsep=-1pt

\bibitem{Arbelaez2011}
P.~Arbelaez, M.~Maire, C.~Fowlkes, and J.~Malik.
\newblock Contour detection and hierarchical image segmentation.
\newblock {\em IEEE Trans. Pattern Anal. Mach. Intell.}, 33(5):898--916, May
  2011.

\bibitem{Benz2002}
M.~Benz, X.~Laboureux, T.~Maier, E.~Nkenke, S.~Seeger, F.~W. Neukam, and
  G.~H{\"{a}}usler.
\newblock The symmetry of faces.
\newblock In {\em Proceedings of the Vision, Modeling, and Visualization
  Conference 2002 {(VMV} 2002), Erlangen, Germany, November 20-22, 2002}, pages
  43--50, 2002.

\bibitem{Besl1992}
P.~J. Besl and N.~D. McKay.
\newblock A method for registration of 3-d shapes.
\newblock {\em {IEEE} Trans. Pattern Anal. Mach. Intell.}, 14(2):239--256,
  1992.

\bibitem{Cai2014}
D.~Cai, P.~Li, F.~Su, and Z.~Zhao.
\newblock An adaptive symmetry detection algorithm based on local features.
\newblock In {\em Visual Communications and Image Processing Conference, 2014
  IEEE}, pages 478--481, Dec 2014.

\bibitem{Chen1992}
Y.~Chen and G.~G. Medioni.
\newblock Object modelling by registration of multiple range images.
\newblock {\em Image Vision Comput.}, 10(3):145--155, 1992.

\bibitem{ConvSymm2016}
M.~Cicconet, V.~Birodkar, M.~Lund, M.~Werman, and D.~Geiger.
\newblock A convolutional approach to reflection symmetry.
\newblock http://arxiv.org/abs/1609.05257, 2016.
\newblock New York.

\bibitem{Cicconet2014CVPR}
M.~Cicconet, K.~Gunsalus, D.~Geiger, and M.~Werman.
\newblock Mirror symmetry histograms for capturing geometric properties in
  images.
\newblock CVPR, 2014.
\newblock Columbus, Ohio.

\bibitem{Funk2016}
C.~Funk and Y.~Liu.
\newblock Symmetry recaptcha.
\newblock In {\em The IEEE Conference on Computer Vision and Pattern
  Recognition (CVPR)}, June 2016.

\bibitem{Gens2014}
R.~Gens and P.~M. Domingos.
\newblock Deep symmetry networks.
\newblock In Z.~Ghahramani, M.~Welling, C.~Cortes, N.~D. Lawrence, and K.~Q.
  Weinberger, editors, {\em Advances in Neural Information Processing Systems
  27}, pages 2537--2545. Curran Associates, Inc., 2014.

\bibitem{Kakarala2013}
R.~Kakarala, P.~Kaliamoorthi, and V.~Premachandran.
\newblock Three-dimensional bilateral symmetry plane estimation in the phase
  domain.
\newblock In {\em 2013 {IEEE} Conference on Computer Vision and Pattern
  Recognition, Portland, OR, USA, June 23-28, 2013}, pages 249--256, 2013.

\bibitem{Kondra2013}
S.~Kondra, A.~Petrosino, and S.~Iodice.
\newblock Multi-scale kernel operators for reflection and rotation symmetry:
  Further achievements.
\newblock In {\em Computer Vision and Pattern Recognition Workshops (CVPRW),
  2013 IEEE Conference on}, pages 217--222, June 2013.

\bibitem{Lee2012}
S.~Lee and Y.~Liu.
\newblock Curved glide-reflection symmetry detection.
\newblock {\em Pattern Analysis and Machine Intelligence, IEEE Transactions
  on}, 34(2):266--278, Feb 2012.

\bibitem{Lewis95}
J.~Lewis.
\newblock Fast template matching.
\newblock In {\em Canadian Image Processing and Pattern Recognition Society,
  Quebec City, Canada}, pages 120--123, 1995.

\bibitem{Li2016}
B.~Li, H.~Johan, Y.~Ye, and Y.~Lu.
\newblock Efficient 3d reflection symmetry detection: {A} view-based approach.
\newblock {\em Graphical Models}, 83:2--14, 2016.

\bibitem{Li2009}
S.~Z. Li and A.~Jain, editors.
\newblock {\em Dynamic Time Warping (DTW)}, pages 231--231.
\newblock Springer US, Boston, MA, 2009.

\bibitem{Liu2013}
J.~Liu, G.~Slota, G.~Zheng, Z.~Wu, M.~Park, S.~Lee, I.~Rauschert, and Y.~Liu.
\newblock Symmetry detection from real world images competition 2013: Summary
  and results.
\newblock CVPR Workshop, 2013.
\newblock Portland, Oregon.

\bibitem{Loy2006}
G.~Loy and J.-O. Eklundh.
\newblock Detecting symmetry and symmetric constellations of features.
\newblock In {\em Proceedings of the 9th European Conference on Computer Vision
  - Volume Part II}, ECCV'06, pages 508--521, Berlin, Heidelberg, 2006.
  Springer-Verlag.

\bibitem{Michaelsen2013}
E.~Michaelsen, D.~Muench, and M.~Arens.
\newblock Recognition of symmetry structure by use of gestalt algebra.
\newblock In {\em Computer Vision and Pattern Recognition Workshops (CVPRW),
  2013 IEEE Conference on}, pages 206--210, June 2013.

\bibitem{Mitra2013}
N.~J. Mitra, M.~Pauly, M.~Wand, and D.~Ceylan.
\newblock Symmetry in 3d geometry: Extraction and applications.
\newblock {\em Comput. Graph. Forum}, 32(6):1--23, 2013.

\bibitem{Munkres1957}
J.~Munkres.
\newblock Algorithms for the assignment and transportation problems.
\newblock {\em Journal of the Society for Industrial and Applied Mathematics},
  5(1):32-38, 1957.

\bibitem{Padia2011}
C.~Padia and N.~Pears.
\newblock A review and characterization of icp-based symmetry plane
  localisation in 3d face data.
\newblock Technical report, Department of Computer Science, University of York,
  UK, 2011.

\bibitem{Pan2006}
G.~Pan, Y.~Wang, Y.~Qi, and Z.~Wu.
\newblock Finding symmetry plane of 3d face shape.
\newblock In {\em 18th International Conference on Pattern Recognition {(ICPR}
  2006), 20-24 August 2006, Hong Kong, China}, pages 1143--1146, 2006.

\bibitem{Patraucean2013}
V.~Patraucean, R.~von Gioi, and M.~Ovsjanikov.
\newblock Detection of mirror-symmetric image patches.
\newblock In {\em Computer Vision and Pattern Recognition Workshops (CVPRW),
  2013 IEEE Conference on}, pages 211--216, June 2013.

\bibitem{Podolak2006}
J.~Podolak, P.~Shilane, A.~Golovinskiy, S.~Rusinkiewicz, and T.~A. Funkhouser.
\newblock A planar-reflective symmetry transform for 3d shapes.
\newblock {\em {ACM} Trans. Graph.}, 25(3):549--559, 2006.

\bibitem{Siddiqi2008}
K.~Siddiqi, J.~Zhang, D.~Macrini, A.~Shokoufandeh, S.~Bouix, and S.~Dickinson.
\newblock Retrieving articulated 3-d models using medial surfaces.
\newblock {\em Mach. Vision Appl.}, 19(4):261--275, May 2008.

\bibitem{Sipiran2014}
I.~Sipiran, R.~Gregor, and T.~Schreck.
\newblock Approximate symmetry detection in partial 3d meshes.
\newblock {\em Comput. Graph. Forum}, 33(7):131--140, 2014.

\bibitem{Sun1997}
C.~Sun and J.~Sherrah.
\newblock 3d symmetry detection using the extended gaussian image.
\newblock {\em {IEEE} Trans. Pattern Anal. Mach. Intell.}, 19(2):164--168,
  1997.

\bibitem{Vetter1994}
T.~Vetter, T.~Poggio, and H.~Blthoff.
\newblock The importance of symmetry and virtual views in three-dimensional
  object recognition.
\newblock {\em Current Biology}, 4(1):18--23, 1994.

\bibitem{Wang2015}
Z.~Wang, Z.~Tang, and X.~Zhang.
\newblock Reflection symmetry detection using locally affine invariant edge
  correspondence.
\newblock {\em Image Processing, IEEE Transactions on}, 24(4):1297--1301, April
  2015.

\bibitem{Wu2011}
J.~Wu, R.~Tse, C.~L. Heike, and L.~G. Shapiro.
\newblock Learning to compute the plane of symmetry for human faces.
\newblock In {\em Proceedings of the 2Nd ACM Conference on Bioinformatics,
  Computational Biology and Biomedicine}, BCB '11, pages 471--474, New York,
  NY, USA, 2011. ACM.

\bibitem{Zheng2015}
Q.~Zheng, Z.~Hao, H.~Huang, K.~Xu, H.~Zhang, D.~Cohen{-}Or, and B.~Chen.
\newblock Skeleton-intrinsic symmetrization of shapes.
\newblock {\em Comput. Graph. Forum}, 34(2):275--286, 2015.

\end{thebibliography}
}

\end{document}